\documentclass{article} 
\usepackage{hyperref}
\usepackage{url}
\usepackage[margin=1.5in]{geometry}

\usepackage{amsthm,amsmath,amssymb}
\usepackage{latexsym}
\usepackage{graphicx}
\usepackage{subfigure}
\usepackage{algorithm}
\usepackage{algorithmic}
\usepackage{color}

\newtheorem{theorem}{Theorem}

\newtheorem{lemma}{Lemma}

\newcommand{\m}{\mathrm{med}_{\lambda} }
\newcommand{\cut}{\mathrm{Cut}}

\newcommand{\real}{\mathbb{R}}
\newcommand{\ones}{{\bf1}}
\newcommand{\x}{\mathbf{x}}

\newcommand{\En}{\mathcal{E}}
\newcommand{\prox}{\text{prox}}
\newcommand{\proj}{\mathrm{proj}}
\newcommand{\R}{\mathbb{R}}
\newcommand{\argmin}{\text{argmin}}
\newcommand{\B}{\mathcal{B}}
\newcommand{\T}{\mathcal{T}}

\title{Multiclass Total Variation Clustering}

\author{ Xavier Bresson\thanks{Department of Computer Science, City University of Hong Kong, Hong Kong ({\tt xbresson@cityu.edu.hk}).},   Thomas Laurent\thanks{Department of Mathematics, University of California Riverside, Riverside CA 92521  ({\tt laurent@math.ucr.edu})}, David Uminsky\thanks{Department of Mathematics, University of San Francisco, San Francisco, CA 94117 ({\tt duminsky@usfca.edu})} and James H. von Brecht\thanks{Department of Mathematics, University of California Los Angeles, Los Angeles CA 90095 ({\tt jub@math.ucla.edu})}  }

\begin{document}

\date{}
\maketitle

\begin{abstract}
Ideas from the image processing literature have recently motivated a new set of clustering algorithms that rely on the concept of total variation. While these algorithms perform well for bi-partitioning tasks, their recursive extensions yield unimpressive results for multiclass clustering tasks. This paper presents a general framework for multiclass total variation clustering that does not rely on recursion. The results greatly outperform previous total variation algorithms and compare well with state-of-the-art NMF approaches.
\end{abstract}
\section{Introduction}
Many clustering models rely on the minimization of an energy over possible partitions of the data set. These discrete optimizations usually pose NP-hard problems, however. A natural resolution of this issue involves relaxing the discrete minimization space into a continuous one to obtain an easier minimization procedure. Many current algorithms, such as spectral clustering methods or non-negative matrix factorization (NMF) methods, follow this relaxation approach. 

A fundamental problem arises when using this approach, however; in general the solution of the relaxed continuous problem and that of the discrete NP-hard problem can differ substantially. In other words, the relaxation is too loose. A \emph{tight} relaxation, on the other hand, has a solution that closely matches the solution of the original discrete NP-hard problem. Ideas from the image processing literature have recently motivated a new set of algorithms \cite{SB09, pro:SzlamBresson10,pro:HeinBuhler10OneSpec,pro:HeinSetzer11TightCheeger,art:BressonTaiChanSzlam12TransLearn, pro:Rang-Hein-constrained, BLUV12, art:BertozziFlenner11DiffuseClassif,MKB12, GMBFP13}
that can obtain tighter relaxations than those used by NMF and spectral clustering. These new algorithms all rely on the concept of total variation. Total variation techniques promote the formation of sharp indicator functions in the continuous relaxation. These functions equal one on a subset of the graph, zero elsewhere and exhibit a non-smooth jump between these two regions. In contrast to the relaxations employed by  spectral clustering and NMF, total variation techniques therefore lead to quasi-discrete solutions that closely resemble the discrete solution of the original NP-hard problem. They provide a promising set of clustering tools for precisely this reason.

Previous  total variation algorithms obtain excellent results for two class partitioning problems   \cite{pro:SzlamBresson10,pro:HeinBuhler10OneSpec,pro:HeinSetzer11TightCheeger, BLUV12} . Until now,  total variation techniques have relied upon a recursive bi-partitioning procedure to handle more than two classes. Unfortunately, these recursive extensions have yet to produce state-of-the art results. This paper presents a general framework for multiclass total variation clustering  that does not rely on a recursive procedure. Specifically, we introduce a new discrete  multiclass clustering  model, its corresponding  continuous relaxation and a new algorithm for optimizing the relaxation. Our approach also easily adapts to handle either unsupervised or transductive clustering tasks. The results significantly outperform previous total variation algorithms and compare well against state-of-the-art approaches \cite{yang2012clustering,yang2012clusteringb,arora2011clustering}. 
We name our approach   Multiclass Total Variation clustering (MTV-clustering).

\section{The Multiclass Balanced-Cut Model}
Given a weighted graph $G = (V,W)$ we let $V=\{ \x_1, \ldots, \x_N\}$ denote the vertex set and $W := \{w_{i,j}\}_{ 1\le i,j\le N}$ denote the non-negative, symmetric similarity matrix. Each entry $w_{ij}$ of $W$ encodes the similarity, or lack thereof, between a pair of vertices. The classical balanced-cut (or, Cheeger cut) \cite{art:Cheeger70RatioCut, book:Chung97Spectral}  asks for a partition of $V = A \cup A^{c}$ into two disjoint sets that minimizes the set energy
\begin{equation}  \label{2classes}
\mathrm{Bal}(A) := \frac{ \cut(A,A^c) } { \min\{ |A| , |A^c| \} } = \frac{ \sum_{\x_i \in A, \x_j \in A^c} w_{ij} } { \min\{ |A| , |A^c| \} }. 
\end{equation}
A simple rationale motivates this model: clusters should exhibit similarity between data points, which is reflected by small values of $\cut(A,A^{c})$, and also form an approximately equal sized partition of the vertex set. Note that $\min\{ |A| , |A^c| \}$ attains its maximum when $|A| = |A^c| = N/2,$ so that for a given value of $\cut(A,A^{c})$ the minimum occurs when $A$ and $A^c$ have approximately  \nolinebreak equal \nolinebreak size. 

We generalize this model to the multiclass setting by pursuing the same rationale. For a given number of classes $R$ we formulate our generalized balanced-cut problem as
\begin{equation*}
\left.\begin{aligned}
&  \hspace{3.5cm}  \text{Minimize } \sum^{R}_{r=1} \;\; \frac{ \cut(A_r,A^c_r) } { \min\{ \lambda |A_r| , |A^c_r| \} }\\
 & \text{over all disjoint partitions $A_{r} \cap A_{s} = \emptyset $, $A_{1} \cup \cdots \cup A_{R} = V$ of the vertex set.}
\end{aligned}
\hspace{0.5cm} \right\}
\qquad \text{(P)}
\end{equation*}
 In this model the parameter $\lambda$ controls the sizes of the sets $A_{r}$ in the partition. Previous work \cite{art:BressonTaiChanSzlam12TransLearn} has used $\lambda = 1$ 
 to obtain a multiclass energy  by a straightforward sum of  the  two-class balanced-cut terms  \eqref{2classes}. While this follows the usual practice, it erroneously attempts to enforce that each set in the partition occupy half of the total number of vertices in the graph. We instead select the parameter $\lambda$ to ensure that each of the classes approximately occupy the appropriate fraction $1/R$ of the total number of vertices. As the maximum of $\min\{ \lambda |A_r| , |A^c_r| \}$ occurs when $\lambda |A_r| = |A^{c}_r| = N - |A_r|,$ we see that $\lambda = R-1$ is the proper choice.

This general framework also easily incorporates \emph{a priori} known information, such as a set of labels for transductive learning. If $L_{r} \subset V$ denotes a set of data points that are \emph{a priori} known to belong to class $r$ then we simply enforce $L_r \subset A_r$ in the definition of an allowable partition of the vertex set. In other words, any allowable disjoint partition $A_{r} \cap A_{s} = \emptyset$, $A_{1} \cup \cdots \cup A_{R} = V$ must also respect the given set of labels.

\section{Total Variation and a Tight Continuous Relaxation}
We derive our continuous optimization by relaxing the set energy (P) 
to the continuous energy
\begin{equation} \label{ACC2}
 \En(F) = \sum_{r=1}^R \frac{   \|f_r\|_{TV} }{\|f_r -\m(f_r)\|}_{1,\lambda}.
\end{equation}
Here $F := [f_1, \ldots, f_R] \in \mathbb{M}_{N \times R}([0,1])$ denotes the $N \times R$ matrix that contains in its columns the relaxed optimization variables associated to the $R$ clusters. A few definitions will help clarify the meaning of this formula. The total variation $\|f\|_{TV}$ of a vertex function $f: V \to \R$  is defined by:
\begin{equation} \label{def:TV}
\|f\|_{TV} = \sum_{i=1}^n w_{ij} |f(\x_i)-f(\x_j)|.
\end{equation}
Alternatively, if we view a vertex function $f$ as a vector $(f(\x_1),\ldots,f(\x_N))^{t} \in \real^{N}$ then we can write
\begin{equation}\label{eq:gradmat}
\|f\|_{TV} := \| Kf \|_{1}. 
\end{equation}
Here $K \in \mathbb{M}_{M \times N}(\R)$ denotes the \emph{gradient matrix} of a graph with $M$ edges and $N$ vertices. Each row of $K$ corresponds to an edge  and each column corresponds to a vertex. For any edge $(i,j)$ in the graph the corresponding row in the matrix $K$ has an entry $w_{ij}$ in the column corresponding to the $i^{{\rm th}}$ vertex, an entry $-w_{ij}$ in the column corresponding to the $j^{ {\rm th}}$ vertex and zeros otherwise. 
 
To make sense of the remainder of \eqref{ACC2} we must introduce the asymmetric $\ell^1$-norm, which 
\begin{equation}
\|f\|_{1,\lambda}= \sum_{i=1}^n | f(\x_i)|_\lambda \qquad \text{ where } \qquad |t|_\lambda=
\begin{cases}
\lambda t & \text{if } t\ge0 \\
- t & \text{if } t<0.
\end{cases} 
\end{equation}
Finally we define the $\lambda$-median (or quantile), denoted $\m(f)$, as:
\begin{align}
&\text{$\m(f) = $ the { $(k+1)^{ {\rm st} }$ } largest value in the range of $f$,}  \text{ where } k= \lfloor  N/(\lambda+1)\rfloor.  
\end{align}
These definitions, as well as the relaxation \eqref{ACC2} itself, were motivated by the following theorem. Its proof, found in the Appendix, relies only on using the three preceeding definitions and some simple algebra.
\begin{theorem} If $f = \ones_A$ is the indicator function of a subset $A \subset V$  then
\begin{equation*}
\frac{   \|f\|_{TV} }{\|f -\m(f)\|}_{1,\lambda} = \frac{2 \; \cut(A ,A^c)}{\min \left\{ \lambda  |A| ,  |A^c| \right\}}. 
\end{equation*}
\end{theorem}
The preceding theorem allows us to restate the original set optimization problem (P) in the equivalent discrete form
\begin{equation*}
\left.\begin{aligned}
&  \hspace{3.5cm}  \text{Minimize }  \sum_{r=1}^R \frac{   \|f_r\|_{TV} }{\|f_r -\m(f_r)\|}_{1,\lambda}\\
 & \text{over  non-zero functions } f_1, \ldots, f_R: V \to \{0,1\} \text{ such that }  f_1+ \ldots + f_R=\ones_V.
\end{aligned}
\hspace{0.5cm} \right\}
\qquad \text{(P')}
\end{equation*}
Indeed, since the non-zero functions $f_r$ can take only two values, zero or one, they must define indicator functions of some nonempty set. The simplex constraint $f_1+ \ldots + f_R=\ones_V$ then guarantees that the sets $A_r := \{ \x_i \in V:  f_r(\x_i) = 1 \}$ form a partition of the vertex set. We obtain the relaxed version (P-rlx) of (P') in the usual manner by allowing $f_{r} \in [0,1]$ to have a continuous range. This yields 
\begin{equation*}
\left.\begin{aligned}
&  \hspace{3.5cm}  \text{Minimize }  \sum_{r=1}^R \frac{   \|f_r\|_{TV} }{\|f_r -\m(f_r)\|}_{1,\lambda}\\
& \text{over   functions } f_1, \ldots, f_R: V \to [0,1] \text{ such that }  f_1+ \ldots + f_R=\ones_V.
\end{aligned}
\hspace{0.5cm} \right\}
\qquad \text{(P-rlx)}
\end{equation*}


The following two points form the foundation on which total variation clustering relies:

\textbf{1} --- As the next subsection details, the total variation terms give rise to \emph{quasi-indicator} functions. That is, the relaxed solutions $[f_1, \ldots, f_R]$ of (P-rlx) mostly take values near zero or one and exhibit a sharp, non-smooth transition between these two regions. Since these quasi-indicator functions essentially take values in the discrete set $\{0,1\}$ rather than the continuous  interval $[0,1]$,  solving  (P-rlx) is almost equivalent to solving either (P) or (P'). In other words, (P-rlx) is a tight relaxation \nolinebreak  of  \nolinebreak  (P).

\textbf{2} --- Both functions $f \mapsto  \|f\|_{TV} $ and $f  \mapsto  \| f -\m(f) \|_{1,\lambda}$ are convex. The simplex constraint in (P-rlx) is also convex. Therefore solving (P-rlx) amounts to minimizing a sum of ratios of convex functions with convex constraints. As the next section details, this fact allows us to use machinery from convex analysis to develop an efficient, novel algorithm for such problems.

\subsection{The Role of Total Variation in the Formation of Quasi-Indicator Functions}

To elucidate the precise role that the total variation itself plays in the formation of quasi-indicator functions, it proves useful to consider a version of (P-rlx) that uses a spectral relaxation in place of the total variation:
\begin{equation*}
\left.\begin{aligned}
&  \hspace{3.5cm}  \text{Minimize }  \sum_{r=1}^R \frac{   \|f_r\|_{ {\rm Lap} } }{\|f_r -\m(f_r)\|}_{1,\lambda}\\
 & \text{over   functions } f_1, \ldots, f_R: V \to [0,1] \text{ such that }  f_1+ \ldots + f_R=\ones_V
\end{aligned}
\hspace{0.5cm} \right\}
\qquad \text{(P-rlx2)}
\end{equation*}
Here $ \|f\|^2_{ {\rm Lap} } = \sum_{i=1}^n w_{ij} |f(\x_i)-f(\x_j)|^2$ denotes the spectral relaxation of $\cut(A,A^c)$; it equals  $\langle f, Lf \rangle$ if $L$ denotes the unnormalized graph Laplacian matrix. Thus problem (P-rlx2) relates to \nolinebreak spectral clustering (and therefore NMF \cite{ding2005equivalence}) with a positivity constraint. Note that the only difference between (P-rlx2) and (P-rlx)  is that the exponent $2$ appears in $\| \cdot \|_{ {\rm Lap} }$ while the exponent $1$ appears in the total variation. This simple difference of exponent has an important consequence for the tightness of the relaxations. Figure 1 presents a simple example that illuminates this difference. If we bi-partition the depicted graph, i.e. a line with $20$ vertices and edge weights $w_{ij} = 1$, then the optimal cut lies between vertex $10$ and vertex $11$ since this gives a perfectly balanced cut.  Figure 1(a) shows the vertex function $f_{1}$ generated by (P-rlx) while figure 1(b) shows the one generated by (P-rlx2). Observe that the solution of the total variation model coincides with the indicator function of the desired cut whereas the the spectral model prefers its smoothed version. Note that both functions in figure 1a) and 1b) have exactly the same total variation $\|f\|_{TV}=|f(\x_1)-f(\x_{2})|+ \cdots + |f(\x_{19})-f(\x_{20})|= f(\x_1)-f(\x_{20})=1$ since both functions are monotonic. The total variation model will therefore prefer the sharp indicator function since it differs more from its $\lambda$-median than the smooth indicator function. Indeed, the denominator $\|f_r -\m(f_r)\|_{1,\lambda}$ is larger for the sharp indicator function than for the smooth one. A different scenario occurs when we replace the exponent one in $\|\cdot\|_{TV}$ by an exponent two, however. As $\|f\|^2_{ {\rm Lap} }=|f(\x_1)-f(\x_{2})|^2+ \cdots + |f(\x_{19})-f(\x_{20})|^2$ and $t^2 < t$ when $t < 1$ it follows that$\|f\|_{ {\rm Lap} }$ is much smaller for the smooth function than for the sharp one. Thus the spectral model will prefer the smooth indicator function despite the fact that it differs less from its $\lambda$-median. We therefore recognize the total variation as the driving force behind the formation of sharp indicator functions. 

\begin{figure}[h!]
\vspace{-.12in}
\hspace{.05in}\includegraphics[scale=.4]{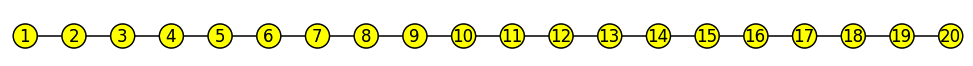} \\
\subfigure[]{ \includegraphics[scale=.35]{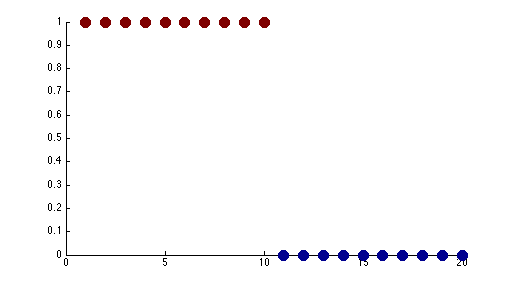} }
\subfigure[]{ \includegraphics[scale=.35]{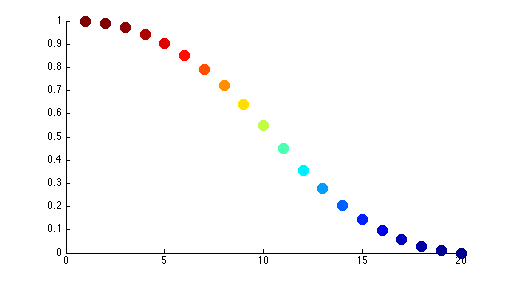}  }
\caption{Top: The graph used for both relaxations. Bottom left: the solution given by the total variation relaxation. Bottom right: the solution given by the spectral relaxation. Position along the $x$-axis $=$ vertex number, height along the $y$-axis $=$ value of the vertex function.}
\end{figure}

This heuristic explanation on a simple, two-class example generalizes to the multiclass case and to real data sets (see figure \ref{fig:realdata}). In simple terms, quasi-indicator functions arise due to the fact that the total variation of a sharp indicator function equals the total variation of a smoothed version  of the same indicator function. The denominator $\|f_r -\m(f_r)\|_{1,\lambda}$ then measures the deviation of these functions from their $\lambda$-median. A sharp indicator function deviates more from its median than does its smoothed version since most of its values concentrate around zero and one. The energy is therefore much smaller for a sharp indicator function than for a smooth indicator function, and consequently the total variation clustering energy always prefers sharp indicator functions to smooth ones. For bi-partitioning problems this fact is well-known. Several previous works have proven that the relaxation is exact in the two-class case; that is, the total variation solution coincides with the solution of the original NP-hard problem \cite{book:Chung97Spectral,pro:SzlamBresson10,BLUV12,pro:BuhlerHein09pLapla}. 

Figure \ref{fig:realdata} illustrates the result of the difference between total variation and NMF relaxations on the data set OPTDIGITS, which contains 5620 images of handwritten numerical digits. Figure \ref{fig:realdata}(a) shows the quasi-indicator function $f_{4}$ obtained by our MTV algorithm while \ref{fig:realdata}(b) shows the function $f_{4}$ obtained from the NMF algorithm of \cite{arora2011clustering}. We extract the portion of each function corresponding to the digits four and nine, then sort and plot the result. The MTV relaxation leads a sharp transition between the fours and the nines while the NMF relaxation leads to a smooth transition.

\begin{figure}[h!]
\centering
\includegraphics[scale=.34]{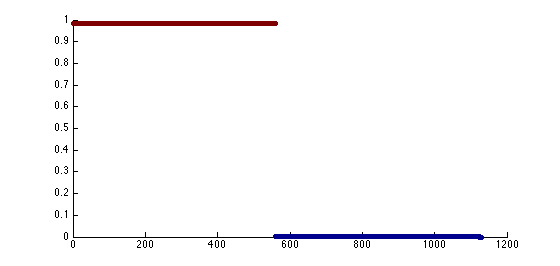} \; \includegraphics[scale=.34]{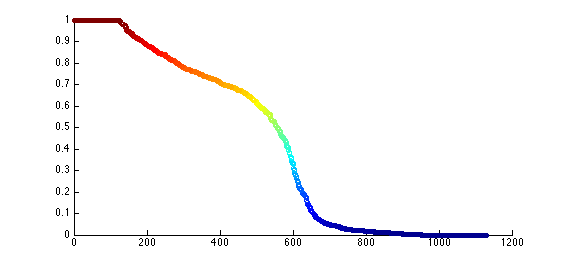}
\caption{Left:  Solution $f_{4}$ from our MTV algorithm plotted over the fours and nines. Right: Solution $f_{4}$ from LSD \cite{arora2011clustering} plotted over the fours and nines.}
\label{fig:realdata}
\end{figure}

\subsection{Transductive Framework}
From a modeling point-of-view, the presence of transductive labels poses no additional difficulty. In addition to the simplex constraint 
\begin{equation}\label{eq:simplex}
F \in \Sigma := \left\{ F \in \mathbb{M}_{N \times R}([0,1]) : f_r(\x_i) \geq 0 , \; \sum^{R}_{r=1} f_r(\x_i) = 1 \right\}
\end{equation}
required for unsupervised clustering we also impose the set of labels as a hard constraint. If $L_{1},\ldots,L_{R}$ denote the $R$ vertex subsets representing the labeled points, so that $\x_i \in L_r$ means $\x_i$ belongs to class $r,$ then we may enforce these labels by restricting $F$ to lie in the subset
\begin{equation}\label{eq:labels}
F \in \Lambda := \left\{ F \in \mathbb{M}_{N \times R}([0,1]) : \forall r, \; (f_1(\x_i),\ldots,f_R(\x_i)) = \mathbf{e}_r \;\; \forall \; \x_i \in L_r \; \right\}.
\end{equation}
Here $\mathbf{e}_{r}$ denotes the row vector containing a one in the $r^{ {\rm th}}$ location and zeros elsewhere. Our model for transductive classification then aims to solve the problem
\begin{equation*}
\left.\begin{aligned}
&  \hspace{0.75cm}  \text{ Minimize }  \sum_{r=1}^R \frac{   \|f_r\|_{TV} }{\|f_r -\m(f_r)\|}_{1,\lambda} \text{over matrices } F \in \Sigma \cap \Lambda.
\end{aligned}
\hspace{0.5cm} \right\}
\qquad \text{(P-trans)}
\end{equation*}
Note that $\Sigma \cap \Lambda$ also defines a convex set, so this minimization remains a sum of ratios of convex functions subject to a convex constraint. Transductive classification therefore poses no additional algorithmic difficulty, either. In particular, we may use the proximal splitting algorithm detailed in the next section for both unsupervised and transductive classification tasks.

\section{Proximal Splitting Algorithm}
This section details our proximal splitting algorithm for finding local minimizers of a sum of ratios of convex functions subject to a convex constraint.  We start by showing in the first subsection that the functions
\begin{equation} \label{TB}
T(f):=\|f\|_{TV} \quad \text{and} \quad B(f):=\|f-\m(f) \ones\|_{1,\lambda} 
\end{equation}
involved in (P-rlx) or (P-trans)  are indeed convex. We also give an explicit formula for a subdifferential of $B$ since our proximal splitting algorithm requires this in explicit form.
 We then summarize a few properties of proximal operators before presenting the algorithm.

\subsection{Convexity, Subgradients and Proximal Operators}
Recall that we may view each function $f: V \to \real$ as a vector in $\real^N$ with $f(\x_i)$ as the $i^{ {\rm th} }$ component of the vector. We may then  view $T$ and $B$ as functions from $\R^{N}$ to $\R$. The next theorem states that both $B$ and $T$ define convex functions on $\R^N$ and furnishes an element $v \in \partial B(f)$ by means of an easily computable formula. The formula for the subdifferential generalizes a related result for the symmetric case \cite{pro:HeinBuhler10OneSpec} to the asymmetric setting.
    \begin{theorem} \label{thm:subdiff}
The functions  $B$ and $T$  are convex.  Moreover, given  $f \in \real^N$ the vector  $v \in \real^N$ defined by
 { \begin{gather}
  v(\x_i)=\begin{cases}
 \lambda & \text{ if } f(\x_i) > \m(f)\\
  \frac{ n^- - \lambda n^+ }{n^0} & \text{ if } f(\x_i) = \m(f) \\
 -1 & \text{ if } f(\x_i) < \m(f)
 \end{cases}   \quad  
  \text{ where} \quad   \begin{cases} n^0 \;= |\{\x_i \in V: f(\x_i)=\m(f)\}| \\  n^-=|\{\x_i \in V: f(\x_i)<\m(f)\}| \\ n^+=|\{ \x_i \in V: f(\x_i)>\m(f)\}|
  \end{cases}
\nonumber
 \end{gather} }
belongs to  $\partial B(f)$. 
\end{theorem}
In the above theorem $\partial B(f)$ denotes the subdifferential of $B$ at $f$ and $v \in \partial B(f)$ denotes a subgradient.  The proof of Theroem 2 can be found in the Appendix.
Given a convex function $A: \real^N \to \real$, the \emph{proximal operator} of $A$  is defined by 
\begin{equation} \label{def:prox}
\prox_{A}(g) := \underset{ f \in \real^N }{\argmin} \; \; A(f) + \frac{1}{2}||f - g||^{2}_{2}.
\end{equation}
If we let $\delta_C$ denote the \emph{barrier function} of the convex set $C$, 
\begin{equation}\label{eq:barrier}
\delta_{C}(f) := \begin{cases} 0 & \text{if} \;\; f \in C \\ +\infty & \text{if} \;\; f \notin C, \end{cases}
\end{equation}
then we easily see that $\prox_{\delta_C}$ is simply  the least-squares projection on $C$:
\begin{equation}
\prox_{\delta_C}(f) = \proj_{C}(f) := \underset{ g \in C}{ \argmin } \; \frac{1}{2} || f - g||^{2}_{2}.
\end{equation}
In this manner the proximal operator defines a mapping from $\real^N$ to $\real^N$ that generalizes the least-squares projection onto a convex set. 

\subsection{The Algorithm}
We can rewrite the problem (P-rlx) or (P-trans) as 
\begin{equation}
  \text{Minimize }  \;\; \delta_{C}(F) + \sum^{R}_{r=1} E(f_r) \quad
 \text{over  all matrices }  F = [f_1, \ldots, f_r] \in  \mathbb{M}_{N \times R}
\end{equation}
where $E(f_r)=T(f_r)/B(f_r)$ denotes the energy of the quasi-indicator function of the $r^{ {\rm th} }$ cluster.  The set $C = \Sigma$ or $C=\Sigma\cap\Lambda$ is the convex subset of  $\mathbb{M}_{N \times R}$ that encodes the simplex constraint \eqref{eq:simplex} or the simplex constraint with labels. The corresponding function $\delta_{C}(F),$ defined in \eqref{eq:barrier}, is the barrier function of the desired set. Beginning from an initial iterate $F^{0} \in C$ we propose the following proximal splitting algorithm:
\begin{equation}\label{eq:algo}
F^{k+1} := \prox_{ \mathcal{T}^{k} + \delta_C } ( F^k + \partial \mathcal{B}^k(F^k) ).
\end{equation}
Here $\mathcal{T}^{k}(F)$ and $\mathcal{B}^{k}(F)$ denote the convex functions 
\begin{equation*}
\mathcal{T}^{k}(F) := \sum^{R}_{r=1}  c_r^k \;  T(f_r) \qquad \mathcal{B}^{k}(F) := \sum^{R}_{r=1} d_r^k  \; B(f_r),
\end{equation*}
the constants $(c_r^k, d_r^k)$ are computed using the previous iterate
$$
c_r^k=\frac{\Delta^{k}}{B(f^k_r)}  \quad   \text{ and } \quad d_r^k=\frac{ \Delta^{k} E(f^k_r)}{B(f^k_r)} 
$$
 and $\Delta^{k}$ denotes the timestep for the current iteration. This choice of the constants $(c_r^k,d_r^k)$ yields
$ \B^k(F^k)=\T^k(F^k)$, and this fundamental property allows us to derive the energy descent estimate:
\begin{theorem}[Estimate of the energy descent] \label{theorem:energy_estimate}  Each of the $F^k$ belongs to $C,$ and if $B^{k}_r \neq 0$ then
\begin{equation}
\sum_{r=1}^R  \frac{B_r^{k+1}}{B_r^k} \left( E_r^k-E_r^{k+1}\right) \ge  \frac{\| F^{k}-F^{k+1} \|^2}{\Delta^k} \label{energy_estimate}
\end{equation}
where $B^k_r,E^k_r$ stand for $B(f^k_r),E(f^k_r)$.
\end{theorem}
Inequality \eqref{energy_estimate} states that the energies of the quasi-indicator functions (as a weighted sum) decrease at every step of the algorithm. It also gives a lower bound for how much these energies decrease. As the algorithm progress and the iterates stabilize the ratio $B_r^{k+1}/B_r^k$ converges to $1$, in which case the sum, rather than a weighted sum, of the individual cluster energies decreases. The proof of Theorem 3 can be found in the Appendix.
 
Our proximal splitting algorithm \eqref{eq:algo} requires two steps. The first step requires computing $G^k=F^k + \partial \mathcal{B}^k(F^k)$, and this is straightforward since theorem  \ref{thm:subdiff} provides the subdifferential of $B$, and therefore of $\mathcal{B}^k$, through an explicit formula. The second step requires computing  $\prox_{ \mathcal{T}^{k} + \delta_C } (G^k)$, which seems daunting at a first glance. Fortunately, minimization problems of this form play an important role in the image processing literature. Recent years have therefore produced several fast and accurate algorithms for computing the proximal operator of the total variation. As $\mathcal{T}^{k} + \delta_{C}$ consists of a weighted sum of total variation terms subject to a convex constraint, we can readily adapt these algorithms to compute the second step of our algorithm efficiently. In this work we use the primal-dual algorithm of \cite{art:ChambollePock11FastPD} with acceleration. This relies on a proper uniformly convex formulation of the proximal minimization, which we detail completely in the Appendix.
 
The primal-dual algorithm we use to compute $\prox_{ \mathcal{T}^{k} + \delta_C } (G^k)$ produces a sequence of approximate solutions by means of an iterative procedure. A stopping criterion is therefore needed to indicate when the current iterate approximates the actual solution $\prox_{ \mathcal{T}^{k} + \delta_C } (G^k)$ sufficiently. Ideally, we would like to terminate $F^{k+1} \approx \prox_{\mathcal{T}^{k} + \delta_C}(G^k)$ in such a manner so that the energy descent property \eqref{energy_estimate} still holds and $F^{k+1}$ always satisfies the required constraints. In theory we cannot guarantee that the energy estimate holds for an inexact solution. We may note, however, that a slightly weaker version of the energy estimate \eqref{energy_estimate}
\begin{equation}
\sum_{r=1}^R \frac{B_r^{k+1}}{B_r^k} \left( E^{k}_r - E^{k+1}_r \right) \ge  (1-\epsilon) \frac{\| F^{k}-F^{k+1} \|^2_F}{\Delta^k} \label{eq:approx_energy_estimate}
\end{equation}
holds after a finite number of iterations of the inner minimization. Moreover, this weaker version still guarantees that the energies of the quasi-indicator functions decrease as a weighted sum in exactly the same manner as before. In this way we can terminate the inner loop adaptively: we solve $F^{k+1} \approx \prox_{\mathcal{T}^{k} + \delta_C}(G^k)$ less precisely when $F^{k+1}$ lies far from a minimum and more precisely as the sequence $\{F^{k}\}$ progresses. Overall this leads to a substantial increase in efficiency of the full algorithm. 

Our implementation of the proximal splitting algorithm also guarantees that $F^{k+1}$ always satisfies the required constraints. We accomplish this task by implementing the primal-dual algorithm in such a way that each inner iteration always satisfies the constraints. This requires computing the projection $\proj_{C}(F)$ exactly at each inner iteration. The overall algorithm remains efficient provided we can compute this projection quickly. When $C = \Sigma$ the algorithm \cite{art:Michelot86Simplex} performs the required projection in at most $R$ steps. When $C = \Sigma \cap \Lambda$ the computational effort actually decreases, since in this case the projection consists of a simplex projection on the unlabeled points and straightforward assignment on the labeled points.

We may now summarize the full algorithm, including the proximal operator computation. In practice we find the choices $\Delta^k = \max\{ B^k_1, \ldots , B^k_R \}$ and any small $\epsilon$ work well, so we present the algorithm with these choices. Recall the matrix $K$ in \eqref{eq:gradmat} denotes the gradient matrix of the graph.
\begin{algorithm}[h!]
\caption{Proximal Splitting Algorithm}
\label{alg-cheeger}
\begin{algorithmic}
\STATE 
Input: $F \in C, P = 0, L = ||K||_{2}, \tau = L^{-1}, \epsilon = 10^{-3}$
\WHILE{loop not converged}

\STATE{//\emph{Perform outer step $G^{k} = F^{k} + \partial \mathcal{B}^{k}(F^k)$ }}
\STATE{ $\Delta = \max_{r}  B(f_r)$ \; $\Delta_{0} = \min_{r} B(f_r)$ \; \;\; $\sigma = \Delta^{2}_{0}( \tau \Delta^2 L^2)^{-1}$ \; \; $\bar{F} = F$}
\STATE{ $D_{E} = \mathrm{diag}\left[ \frac{E(f_1)}{B(f_1)},\ldots , \frac{E(f_R)}{B(f_R)} \right]$ \;\; $D_{B} = \mathrm{diag}\left[\frac{\Delta}{B(f_1)},\ldots , \frac{\Delta}{B(f_R)} \right]$}
\STATE{ $V = \Delta [ \partial B(f_1), \ldots , \partial B(f_R) ] D_{E} $ (using theorem \ref{thm:subdiff}) }
\STATE{ $G = F + V $ }
\STATE{//\emph{Perform $F^{k+1} \approx \prox_{\mathcal{T}^{k} + \delta_C}(G^k)$ until energy estimate holds}}
\WHILE{ \eqref{eq:approx_energy_estimate} fails }

\STATE{ $\tilde{P} = P + \sigma K \bar{F} D_{B} $ \; \; $P = \tilde{P}/ \max\{ |\tilde{P}|,1 \}$ (both operations entriwise) \; \; $F_{ {\rm old} } = F$ } 
\STATE{ $\tilde{F} = F - \tau K^{t} P D_{B}$ \;\; $F = (\tilde{F} + \tau G)/(1+\tau)$ \; \; $F = \proj_{C}(F)$ }
\STATE{ $\theta = 1/\sqrt{ 1 + 2\tau }$ \; \; $\tau = \theta \tau$ \; \; $\sigma = \sigma/\theta$ \; \; $ \bar{F} = (1+\theta)F - \theta F_{ {\rm old} }$ }

\ENDWHILE

\ENDWHILE
\end{algorithmic}
\end{algorithm}

\section{Numerical Experiments}\label{sec:numerics}
We now demonstrate the MTV algorithm for unsupervised and transductive clustering tasks. We selected six standard, large-scale data sets as a basis of comparison. We obtained the first data set (4MOONS) and its similarity matrix from \cite{art:BressonTaiChanSzlam12TransLearn} and the remaining five data sets and matrices (WEBKB4, OPTDIGITS, PENDIGITS, 20NEWS, MNIST) from \cite{yang2012clustering}. The 4MOONS data set contains 4K points while the remaining five contain 4.2K, 5.6K, 11K, 20K and 70K points, respectively.

Our first set of experiments compares our MTV algorithm against other unsupervised approaches. We compare against two previous total variation algorithms \cite{pro:HeinBuhler10OneSpec, BLUV12}, which rely on recursive bi-partitioning, and two top NMF algorithms \cite{arora2011clustering,yang2012clustering}. We use the normalized Cheeger cut versions of \cite{pro:HeinBuhler10OneSpec} and \cite{BLUV12} with default parameters. We used the code available from \cite{yang2012clustering} to test each NMF algorithm. All non-recursive algorithms (LSD \cite{arora2011clustering}, NMFR  \cite{yang2012clustering}, MTV) received two types of initial data: (a) the deterministic data used in \cite{yang2012clustering}; (b) a random procedure leveraging normalized-cut \cite{art:ShiMalik00NCut}. Procedure (b) first selects one data point uniformly at random from each computed NCut cluster, then sets $f_r$ equal to one at the data point drawn from the $r^{{\rm th}}$ cluster and zero otherwise. We then propagate this initial stage by replacing each $f_{r}$ with $(I + L)^{-1}f_r$ where $L$ denotes the unnormalized graph Laplacian. Finally, to aid the NMF algorithms, we add a small constant $0.2$ to the result (each performed better than without adding this constant). For MTV we use (a) and 30 random trials of (b) then report the cluster purity of the solution with the lowest discrete energy (P). We then use each NMF with exactly the same initial conditions and report simply the highest purity achieved over all 31 runs. This biases the results in favor of the NMF algorithms. Due to the non-convex nature of these algorithms, the random initialization gave the best results and significantly improved on previously reported results of LSD in particular. We allowed each non-recursive algorithm 10000 iterations using initial condition (a) while each trial of (b) performed 2000 iterations. The following table reports the results.
\begin{table}[h!]
\centering
\begin{tabular}{c|c|c|c|c|c|c}
Alg/Data & 4MOONS & WEBKB4 & OPTDIGITS & PENDIGITS & 20NEWS & MNIST \\
\hline	
NCC-TV \cite{BLUV12} & 88.75  & 51.76  & 95.91 & 73.25 & 23.20 & 88.80 \\
1SPEC \cite{pro:HeinBuhler10OneSpec}&  73.92 & 39.68 & 88.65 & 82.42 & 11.49 & 88.17 \\
LSD \cite{arora2011clustering}&  99.40 & 54.50 & 97.94 & 88.44 & 41.25 & 95.67 \\
NMFR \cite{yang2012clustering}  &  77.80 & 64.32 & 97.92 & 91.21 & 63.93 & 96.99 \\
MTV  &  99.53 & 59.15 & 98.29 & 89.06 & 39.40 & 97.60 \\
\hline		
\end{tabular}
\label{tab:unsuper}
\end{table}
Our next set of experiments demonstrate our algorithm in a transductive setting. For each data set we randomly sample either one label per class or a percentage of labels per class from the ground truth. We then run ten trials of initial condition (b) (propagating all labels instead of one) and report the purity of the lowest energy solution as before along with the average computational time (for simple MATLAB code running on a standard desktop) of the ten runs. We terminate the algorithm once the relative change in energy falls below $10^{-4}$ between outer steps of algorithm 1. The table below reports the results. Note that for well-constructed graphs (such as MNIST), our algorithm performs remarkably well with only one label per class.
\begin{table}[h!]
\centering
\begin{tabular}{c|c|c|c|c|c|c}
Labels & 4MOONS & WEBKB4 & OPTDIGITS & PENDIGITS & 20NEWS & MNIST \\
\hline	
1 &    99.55/ 3.0s           &  56.58/ 1.8s  & 98.29/ 7s & 89.17/ 14s & 50.07/ 52s & 97.53/ 98s \\
1$\%$ &     99.55/ 3.1s          &     58.75/ 2.0s   &  98.29/ 4s & 93.73/ 9s &  61.70/ 54s  & 97.59/ 54s \\
2.5$\%$  &      99.55/ 1.9s       &  57.01/ 1.7s        & 98.35/ 3s & 95.83/ 7s & 67.61/ 42s & 97.72/ 39s \\
5$\%$  &         99.53/ 1.2s      &       58.34/ 1.3s    &   98.38/ 2s            & 97.98/ 5s & 70.51/ 32s & 97.79/ 31s \\
10$\%$ &          99.55/ 0.8s       &        62.01/ 1.2s      &      98.45/ 2s       &  98.22/ 4s  & 73.97/ 25s & 98.05/ 25s \\
\hline		
\end{tabular}
\label{tab:trans}
\end{table}

Our non-recursive MTV algorithm vastly outperforms the two previous recursive total variation approaches and also comparse well with state-of-the-art NMF approaches. Each of MTV, LSD and NMFR perform well on manifold data sets such as MNIST but NMFR tends to perform best on noisy, non-manifold data sets. This results from the fact that NMFR uses a costly graph smoothing technique while our algorithm and LSD do not. We plan to incorporate such improvements into the total variation framework in future work. Lastly, we found procedure (b) can help overcome the lack of convexity inherent in many clustering approaches. We plan to pursue a more principled and efficient initialization along these lines in the future as well. Overall, our total variation framework therefore presents a promising alternative to NMF methods due to its strong mathematical foundation and tight relaxation. 


\section{Appendix}

\subsection{Proofs of Theorems}

\begin{theorem} If $f = \ones_A$ is the indicator function of a subset $A \subset V$  then
\begin{equation*}
\frac{   \|f\|_{TV} }{\|f -\m(f)\|}_{1,\lambda} = \frac{2 \; \cut(A ,A^c)}{\min \left\{ \lambda  |A| ,  |A^c| \right\}}. 
\end{equation*}
\end{theorem}
\begin{proof}
The fact that $ \|f\|_{TV} =2 \; \cut(A ,A^c)$ follows directly from the definition of the total variation. Indeed, a straightforward computation shows
\begin{equation*}
\| f \|_{TV} = \sum_{ \x_{i} \in A } \sum^{N}_{j=1} w_{ij}|1 - f(\x_{j})| + \sum_{ \x_{i} \in A^{c} } \sum^{N}_{j=1} w_{ij}|f(\x_{j})| = \sum_{ \x_{i} \in A } \sum_{\x_j \in A^{c} } w_{ij} +  \sum_{ \x_{i} \in A^{c} } \sum_{\x_j \in A } w_{ij}.
\end{equation*}
Thus $ \|f\|_{TV} =2 \; \cut(A ,A^c)$ as $W$ is symmetric. Let $B(f) := {\|f -\m(f)\|}_{1,\lambda}$. To show that $B(f) = \min \left\{ \lambda  |A| ,  |A^c| \right\}$, suppose first that  $\lambda  |A|  \leq  |A^c|$. This inequality implies $\lambda |A| \leq N - |A|,$ or  equivalently that $|A| \leq  N/(1+\lambda)$. Thus $|A| \leq k := \lfloor N/(1+\lambda) \rfloor,$ and since $f = \ones_{A}$ for $|A| \leq k$ it follows immediately that the  $(k+1)^{{ \rm st } }$  largest entry in the range of  $f$ equals zero. Thus $\m(f) = 0$ by definition. A direct computation then yields that $B(f)=\sum_{i \in V} | f(\x_i)|_\lambda = \lambda |A|$. In the converse case, the fact that  $|A^c| < \lambda  |A|  $   implies  $|A| > N/(1+\lambda) \geq k$. Thus $|A| \geq k+1$ and $\m(f)=1$. Direct computation then shows that $B(f)=\sum_{i \in V} | f(\x_i)  - 1|_\lambda   =   |A^c|$ as claimed.
\end{proof}

\begin{lemma}
Let $h \in \R^{N}$ and suppose $v \in \R^{N}$ satisfies
\begin{equation}
v(\x_{i}) \in \begin{cases}
\lambda & \text{if } h(\x_{i}) > 0 \\
[-1,\lambda] &   \text{if } h(\x_{i}) = 0 \\
-1 & \text{if } h(\x_{i}) < 0.
\end{cases}
\end{equation}
Then $v \in \partial \|h \|_{1,\lambda}$.
\end{lemma}
\begin{proof}
Note that $|h(\x_i)|_{\lambda} = v(\x_i) h(\x_i)$ for each $\x_{i},$ so that for arbitrary $g \in \R^{N}$ and each $\x_i$ the inequality
\begin{equation*}
|g(\x_i)|_{\lambda} - |h(\x_{i})|_{\lambda} \geq v(\x_i)\left( g(\x_i) - h(\x_i) \right)
\end{equation*}
holds. Summing both sides over all $\x_i \in V$ then gives the claim.
\end{proof}

\begin{theorem} \label{thm:subdiff}
The functions  $B$ and $T$  are convex.  Moreover, given  $f \in \real^N$ the vector  $v \in \real^N$ defined by
 \begin{gather}
  v(\x_i)=\begin{cases}
 \lambda & \text{ if } f(\x_i) > \m(f)\\
  \frac{n^- - \lambda n^+ }{n^0} & \text{ if } f(\x_i) = \m(f) \\
 -1 & \text{ if } f(\x_i) < \m(f)
 \end{cases}   \quad  
  \text{ where} \quad   \begin{cases} n^0 \;= |\{\x_i \in V: f(\x_i)=\m(f)\}| \\  n^-=|\{\x_i \in V: f(\x_i)<\m(f)\}| \\ n^+=|\{ \x_i \in V: f(\x_i)>\m(f)\}|
  \end{cases}
\nonumber
 \end{gather}
belongs to  $\partial B(f)$. 
\end{theorem}

\begin{proof} 
The convexity of $T(f)$ follows directly from its definition and a straightforward computation using the definition of convexity. Due to the continuity $B(f),$ to show convexity it suffices to establish the existence of a subdifferential at every point.

To this end note that $\m(f) \in \mathrm{range}(f),$ so that in particular $n^0 \ge 1$ by definition. Let $1 \leq k := \lfloor N/(1+\lambda) \rfloor < N$ denote that entry of $f$ so that $f(\x_{k}) = \m(f)$. By definition of $\m(f)$ there exist at most $k$ elements of $f$ larger than $\m(f)$, so that $n^{+} \leq k \leq N/(1+\lambda)$. As $N = n^{-} + n^{0} + n^{+}$ this implies $\frac{ \lambda n^+ - n^- }{n^0} \leq 1$. Similarly there exist at most $N-(k+1)$ elements of $f$ smaller than $\m(f)$, so that $n^{-} \leq N-(k+1) \leq N - N/(1+\lambda)$. The fact that $N = n^{-} + n^{0} + n^{+}$ then implies $\frac{ n^{-} - \lambda n^{+} }{n^0 } \leq \lambda$. Combining this with the previous inequality yields $-1 \leq \frac{ n^- - \lambda n^+ }{n^0} \leq \lambda.$

Put $h := f - \m(f) \ones,$ and note that the vector $v$ defined above satisfies $v \in \partial \|h\|_{1,\lambda}$ by the preceeding lemma. Thus for any $g \in \R^{N}$ it holds that
\begin{equation*}
|| g - \m(g) \ones ||_{1,\lambda} - ||f - \m(f) \ones||_{1,\lambda} \geq \langle v , g - f + (\m(f) - \m(g) ) \ones \rangle
\end{equation*}
by definition of the subdifferential. Note also that $\langle v, \ones \rangle = 0,$ so that in fact
\begin{equation*}
B(g) - B(f) = || g - \m(g) \ones ||_{1,\lambda} - ||f - \m(f) \ones||_{1,\lambda} \geq \langle v , g - f \rangle 
\end{equation*}
for $g \in \R^{N}$ arbitrary. Thus $v \in \partial B(f)$ by definition of the subdifferential. In particular $\partial B(f)$ is always non-empty, so $B(f)$ is convex.
\end{proof}

\begin{theorem}[Estimate of the energy descent] \label{theorem:energy_estimate}  Each of the $F^k$ belongs to $C,$ and if $B^{k}_r \neq 0$ then 
\begin{equation}
\sum_{r=1}^R  \frac{B_r^{k+1}}{B_r^k} \left( E_r^k-E_r^{k+1}\right) \ge  \frac{\| F^{k}-F^{k+1} \|^2}{\Delta^k} \label{energy_estimate}
\end{equation}
where $B^k_r,E^k_r$ stand for $B(f^k_r),E(f^k_r)$.
\end{theorem}

\begin{proof}
Let
$V^k \in  \partial \B^k(F^k)$. Then by definition of the subdifferential it follows that
\begin{equation}
 \B^k(F^{k+1}) \ge  \B^k(F^k) + \langle  F^{k+1}-F^k , V^k \rangle. \label{B_ineq}
 \end{equation}
As $F^{k+1} =  \text{prox}_{\T^k+  \delta_C}( F^k  +V^k) $ the definition of the proximal operator implies that $F^{k+1} \in C$ and that also
\begin{equation*}
F^k + V^k - F^{k+1}  \in  \partial (\T^k+ \delta_C)(F^{k+1}).
\end{equation*}
The definition of the subdifferential and the fact that $\delta_{C}(F^k) = \delta_{C}(F^{k+1}) = 0$ then combine to imply
\begin{align}
\T^k(F^k) & \ge \T^k(F^{k+1})+ \langle F^{k}-F^{k+1} , F^{k}+V^k -F^{k+1} \rangle \nonumber \\
& =  \T^k(F^{k+1})+ \| F^{k}-F^{k+1} \|^2 +  \langle F^{k}-F^{k+1} , V^k  \rangle  \label{T_ineq}
\end{align}
Adding \eqref{B_ineq} and \eqref{T_ineq} yields
$$
\T^k(F^{k})+\B^k(F^{k+1}) \ge \T^k(F^{k+1})+\B^k(F^{k})+\| F^{k}-F^{k+1} \|^2,
$$
or equivalently that $\B^k(F^{k+1}) \ge \T^k(F^{k+1})+\| F^{k}-F^{k+1} \|^2$ since $ \B^k(F^k)=\T^k(F^k)$ by construction. Expanding this last inequality shows
\begin{equation*}
\sum^{R}_{r=1} \frac{\Delta^k}{B^k_r} \left( E^{k}_r B^{k+1}_r - T^{k+1}_r \right) \geq \| F^{k}-F^{k+1} \|^2,
\end{equation*}
which yields the claim after by $B^{k+1}_r$ in each term of the summation.
\end{proof}

\subsection{Primal-Dual Formulation}

Consider the minimization
\begin{equation*}
F^{k+1} : =\prox_{ \mathcal{T}^k + \delta_{C} } ( G^k ). 
\end{equation*}
We may write this as the saddle-point problem
\begin{equation*}
\min_{ u \in \R^{NR} } \max_{p \in \R^{MR} } \; \langle p , \mathcal{K} u \rangle + G(u) - F^{*}(p).
\end{equation*}
Here the vector $u = (f_1,\ldots,f_{R})^{t}$ is a ``vectorized'' version of $F$ and the matrix $\mathcal{K}$ denotes the block diagonal matrix
\begin{equation*}
\mathcal{K} := \mathrm{blkdiag}\left(\frac{\Delta^k}{B^k_1} K,\ldots,\frac{\Delta^k}{B^k_R} K \right)
\end{equation*}
where $K$ is the gradient matrix of the graph. We define the convex function $G(u)$ as
\begin{equation*}
G(u) := \frac{1}{2} \sum^{R}_{r=1} ||f_r - g^k_r ||^2 + \delta_{C}(u),
\end{equation*}
where $\delta_{C}$ denotes the barrier function of the convex set $C$ (either the simplex or simplex with labels) as before. The convex function $F^{*}(p)$ denotes the barrier function of the $l^{\infty}$ unit ball, so that
\begin{equation*}
F^{*}(p) = \begin{cases}
0 & \text{if} \qquad |p_{i}| \leq 1 \; \; \forall \; 1 \leq i \leq MR \\
+\infty & \text{otherwise}.
\end{cases}
\end{equation*}
Note also that $G(u)$ is uniformly convex, in that if $v \in \partial G(u)$ denotes any subdifferential then for any $u' \in \R^{NR}$ the inequality
\begin{equation*}
G(u') - G(u) \geq \langle v , u' - u \rangle + \frac{1}{2}||u - u'||^{2}
\end{equation*}
holds. We may therefore apply algorithm 2 of \cite{art:ChambollePock11FastPD} with $\gamma = 1$ with to solve the saddle-point problem. This algorithm consists in the iterations
\begin{align*}
p^{n+1} &= \prox_{ \sigma^n F^{*} }( p^n + \sigma^{n} \mathcal{K} \bar{u}^n ) \\
u^{n+1} &= \prox_{ \tau^n G }( u^n - \tau^{n} \mathcal{K}^{t} p^{n+1} ) \\
\theta^{n} &= \frac{1}{ \sqrt{ 1 + 2 \tau^{n} } } \quad \tau^{n+1} = \theta^n \tau^n \quad \sigma^{n+1} = \sigma^{n} / \theta^n \\
\bar{u}^{n+1} &= u^{n+1} + \theta^{n}( u^{n+1} - u^n )
\end{align*}
and converges provided the inequality $\sigma^0 \leq (\tau^0 ||\mathcal{K}||^2_2)^{-1}$ holds for the initial timesteps. We may compute the inner proximal operators analytically to find
\begin{equation*}
( \prox_{ \sigma^n F^{*} }(z) )_{i} = z_i/\max\{1,|z_i|\} \quad \forall \; 1 \leq i \leq MR,
\end{equation*}
and by completing the square that
\begin{equation*}
\prox_{ \tau^n G }(z) = \proj_{C}\left( \frac{z + \tau^n g}{1+\tau^n} \right),
\end{equation*}
where $g = (g^k_1,\ldots,g^k_R)^t$ denotes $G^k$ in vectorized form. The inner loop of algorithm 1 then follows by re-writing these computations in matrix form.









{


\bibliographystyle{plain}

\bibliography{bib_nips}

\begin{thebibliography}{10}

\bibitem{arora2011clustering}
Raman Arora, M~Gupta, Amol Kapila, and Maryam Fazel.
\newblock Clustering by left-stochastic matrix factorization.
\newblock In {\em International Conference on Machine Learning (ICML)}, pages
  761--768, 2011.

\bibitem{art:BertozziFlenner11DiffuseClassif}
A.~Bertozzi and A.~Flenner.
\newblock {Diffuse Interface Models on Graphs for Classification of High
  Dimensional Data}.
\newblock {\em Multiscale Modeling and Simulation}, 10(3):1090--1118, 2012.

\bibitem{BLUV12}
X.~Bresson, T.~Laurent, D.~Uminsky, and J.~von Brecht.
\newblock Convergence and energy landscape for cheeger cut clustering.
\newblock In {\em Advances in Neural Information Processing Systems (NIPS)},
  pages 1394--1402, 2012.

\bibitem{art:BressonTaiChanSzlam12TransLearn}
X.~Bresson, X.-C. Tai, T.F. Chan, and A.~Szlam.
\newblock {Multi-Class Transductive Learning based on $\ell^1$ Relaxations of
  Cheeger Cut and Mumford-Shah-Potts Model}.
\newblock {\em UCLA CAM Report}, 2012.

\bibitem{pro:BuhlerHein09pLapla}
T.~B\"{u}hler and M.~Hein.
\newblock {Spectral Clustering Based on the Graph p-Laplacian}.
\newblock In {\em International Conference on Machine Learning (ICML)}, pages
  81--88, 2009.

\bibitem{art:ChambollePock11FastPD}
A.~Chambolle and T.~Pock.
\newblock {A First-Order Primal-Dual Algorithm for Convex Problems with
  Applications to Imaging}.
\newblock {\em Journal of Mathematical Imaging and Vision}, 40(1):120--145,
  2011.

\bibitem{art:Cheeger70RatioCut}
J.~Cheeger.
\newblock {A Lower Bound for the Smallest Eigenvalue of the Laplacian}.
\newblock {\em Problems in Analysis}, pages 195--199, 1970.

\bibitem{book:Chung97Spectral}
F.~R.~K. Chung.
\newblock {\em {Spectral Graph Theory}}, volume~92 of {\em CBMS Regional
  Conference Series in Mathematics}.
\newblock Published for the Conference Board of the Mathematical Sciences,
  Washington, DC, 1997.

\bibitem{ding2005equivalence}
Chris Ding, Xiaofeng He, and Horst~D Simon.
\newblock On the equivalence of nonnegative matrix factorization and spectral
  clustering.
\newblock In {\em Proc. SIAM Data Mining Conf}, number~4, pages 606--610, 2005.

\bibitem{GMBFP13}
C.~Garcia-Cardona, E.~Merkurjev, A.~L. Bertozzi, A.~Flenner, and A.~G. Percus.
\newblock Fast multiclass segmentation using diffuse interface methods on
  graphs.
\newblock {\em Submitted}, 2013.

\bibitem{pro:HeinBuhler10OneSpec}
M.~Hein and T.~B\"{u}hler.
\newblock {An Inverse Power Method for Nonlinear Eigenproblems with
  Applications in 1-Spectral Clustering and Sparse PCA}.
\newblock In {\em Advances in Neural Information Processing Systems (NIPS)},
  pages 847--855, 2010.

\bibitem{pro:HeinSetzer11TightCheeger}
M.~Hein and S.~Setzer.
\newblock {Beyond Spectral Clustering - Tight Relaxations of Balanced Graph
  Cuts}.
\newblock In {\em Advances in Neural Information Processing Systems (NIPS)},
  2011.

\bibitem{MKB12}
E.~Merkurjev, T.~Kostic, and A.~Bertozzi.
\newblock An mbo scheme on graphs for segmentation and image processing.
\newblock {\em UCLA CAM Report 12-46}, 2012.

\bibitem{art:Michelot86Simplex}
C.~Michelot.
\newblock {A Finite Algorithm for Finding the Projection of a Point onto the
  Canonical Simplex of Rn}.
\newblock {\em Journal of Optimization Theory and Applications},
  50(1):195--200, 1986.

\bibitem{pro:Rang-Hein-constrained}
S.~Rangapuram and M.~Hein.
\newblock {Constrained 1-Spectral Clustering}.
\newblock In {\em International conference on Artificial Intelligence and
  Statistics (AISTATS)}, pages 1143--1151, 2012.

\bibitem{art:ShiMalik00NCut}
J.~Shi and J.~Malik.
\newblock {Normalized Cuts and Image Segmentation}.
\newblock {\em IEEE Transactions on Pattern Analysis and Machine Intelligence
  (PAMI)}, 22(8):888--905, 2000.

\bibitem{SB09}
A.~Szlam and X.~Bresson.
\newblock A total variation-based graph clustering algorithm for cheeger ratio
  cuts.
\newblock {\em UCLA CAM Report 09-68}, 2009.

\bibitem{pro:SzlamBresson10}
A.~Szlam and X.~Bresson.
\newblock Total variation and cheeger cuts.
\newblock In {\em International Conference on Machine Learning (ICML)}, pages
  1039--1046, 2010.

\bibitem{yang2012clustering}
Zhirong Yang, Tele Hao, Onur Dikmen, Xi~Chen, and Erkki Oja.
\newblock Clustering by nonnegative matrix factorization using graph random
  walk.
\newblock In {\em Advances in Neural Information Processing Systems (NIPS)},
  pages 1088--1096, 2012.

\bibitem{yang2012clusteringb}
Zhirong Yang and Erkki Oja.
\newblock Clustering by low-rank doubly stochastic matrix decomposition.
\newblock In {\em International Conference on Machine Learning (ICML)}, 2012.

\end{thebibliography}

\end{document}